\documentclass[twoside]{article}

\usepackage[round]{natbib}

\usepackage[accepted]{aistats2026}
\usepackage{hyperref}
\usepackage{url}

\usepackage[utf8]{inputenc} 
\usepackage[T1]{fontenc}    
\usepackage{booktabs}       
\usepackage{amsfonts}       
\usepackage{nicefrac}       
\usepackage{microtype}      
\usepackage{graphicx}
\usepackage{pifont}
\usepackage{caption}
\usepackage{subcaption}
\usepackage{makecell}
\usepackage{adjustbox}
\usepackage{colortbl}        

\usepackage{multirow}
\usepackage{CJKutf8}
\usepackage[dvipsnames, table]{xcolor}

\usepackage{amsmath}
\usepackage{amssymb}
\usepackage{mathtools}
\usepackage{amsthm}
\usepackage{float}
\usepackage{hyperref}
\usepackage[most]{tcolorbox}

\theoremstyle{plain}
\newtheorem{theorem}{Theorem}[section]

\begin{document}
%

%

\twocolumn[

\aistatstitle{Linear Reasoning vs. Proof by Cases: Obstacles for Large Language Models in FOL Problem Solving}

\aistatsauthor{ Yuliang Ji \And Fuchen Shen \And  Jian Wu \And  Qiujie Xie \And Yue Zhang }

\aistatsaddress{ Nanjing University of \\ Science and Technology  \And  Westlake University \And Westlake University \And Zhejiang University \\ Westlake University  \And Westlake University } ]

\begin{abstract}
To comprehensively evaluate the mathematical reasoning capabilities of Large Language Models (LLMs), researchers have introduced abundant mathematical reasoning
datasets. 
However, most existing datasets primarily focus on linear reasoning, neglecting other parts such as proof by contradiction and proof by cases, which are crucial for investigating LLMs' reasoning abilities.
To address this limitation, 
we first introduce a novel first-order logic (FOL) dataset named \textbf{PC-FOL}, 
annotated by professional mathematicians, focusing on case-based reasoning problems. All instances in
this dataset are equipped with a manually written natural language proof
, clearly distinguishing it from conventional linear reasoning datasets.
Our experimental results over leading LLMs demonstrate \textbf{a substantial performance gap between linear reasoning and case-based reasoning problems.} To further investigate this phenomenon, we provide \textbf{a theoretical analysis} grounded in graphical model, which provides an explanation for the observed disparity between the two types of reasoning problems.
We hope this work can 
reveal the core challenges in the field of automated natural language mathematical proof generation, paving the way for future research.
\end{abstract}

\section{Introduction}

By utilizing massive training datasets and computational resources, LLMs have shown potential in solving mathematical problems \citep{ahn-etal-2024-large, kojima2023largelanguagemodelszeroshot}, where 
mathematical logic establishes an essential role in structuring mathematical proofs and reasoning  \citep{principles_of_mathematical_logic_Hilbert_1928}.
The core of mathematical logic involves propositional logic and first-order logic (\textbf{FOL}), where FOL problems need to derive conclusions from given premises, demanding logical reasoning skills and the capability to process complex logical relationships described in natural language \citep{2021antheropics, parmar-etal-2024-logicbench}.

To comprehensively evaluate the mathematical reasoning capabilities of LLMs, scientists have introduced many datasets \citep{Kaliszyk2017HolStep, GSM8K, mitra2024orcamath200k, gsm-symbolic, welleck2021naturalproofs, yang2023leandojo}.
For example, researchers have evaluated the capabilities of their models on the GSM8K \citep{GSM8K} dataset, which contains both computational and word problems, requiring models to possess logical reasoning and calculation abilities to get correct answers. 
However, limited FOL reasoning datasets \citep{MathQA} have been proposed to evaluate the FOL reasoning abilities of LLMs, and each of them has its own shortcomings.
As shown in Table \ref{evaluate_dataset}, RuleTaker \citep{ruletaker2020}, LogicNLI \citep{logicnli2021}, ProntoQA \citep{prontoQA} are datasets that include FOL examples, but they do not provide natural language proofs. 
FOLIO \citep{folio2024} and P-FOLIO \citep{p-folio2024} are among the earliest datasets to propose standard FOL problems based on real-world stories, written in natural language by human annotators. However, FOLIO does not provide corresponding answers, and P-FOLIO includes only step-by-step reasoning chains without natural language explanations. 

More importantly, although existing FOL datasets have made progress in terms of coverage and linguistic diversity, \textbf{they generally lack systematic annotation of proof strategies, particularly whether a problem requires the use of proof by cases}, a widely employed reasoning technique in discrete mathematics \citep{proofofkeplerconjecture2005, JI2019230}. From the perspective of reasoning methodology, natural language FOL problems can be broadly categorized into two types~(Table \ref{table_of_difference_between_linear_and_case}): \textbf{linear-reasoning} and \textbf{proof-by-cases}. Problems in the linear-reasoning category exhibit a single, sequential reasoning path, where conclusions can be derived through straightforward step-by-step inference. In contrast, proof-by-cases problems require partitioning the reasoning process into multiple scenarios, analyzing each case independently, and subsequently integrating the outcomes to determine the final truth value. Thus, it is necessary to evaluate LLMs’ FOL reasoning abilities across these two types of problems, since the proof processes they involve are fundamentally different.

\begin{table}
  \caption{Comparison of BS-FOL with other datasets related to mathematical logic reasoning. ``Standard FOL'' represents whether the data instances are written in formal FOL. ``NL Instance'' represents whether the dataset is written in natural language. ``Reasoning Chains'' shows whether the dataset explains the answer in the form of reasoning chains. ``NL proofs'' represents whether the dataset gives natural language proofs. ``Linear/Case Label'' represents whether the data instance is supported with a label of FOL type.}
  \label{evaluate_dataset}
  \centering
  \begin{adjustbox}{width=\linewidth}
    \begin{tabular}{c|ccccc}
    \toprule
    Dataset Name     &Size & \makecell[l]{Standard \\ FOL} & \makecell[l]{NL \\ Instance} &\makecell[l]{NL \\ proofs} & \makecell[l]{Linear/Case \\ Label} \\
    \midrule
    GSM8K(\citeyear{GSM8K}) & 8.5k & \ding{53} & \checkmark  & \checkmark & N/A \\
    MathQA(\citeyear{MathQA}) & 37k & \ding{53} & \checkmark  & \checkmark & N/A \\
    RuleTaker(\citeyear{ruletaker2020}) & 500k & \checkmark & \checkmark  & \ding{53} & \ding{53} \\
    LogicNLI(\citeyear{logicnli2021}) & 20k & \checkmark & \checkmark  & \ding{53} & \ding{53} \\
    ProntoQA(\citeyear{prontoQA}) & 46k & \checkmark & \checkmark  & \ding{53} & \ding{53} \\
    FOLIO(\citeyear{folio2024})       & 1204 & \checkmark  & \checkmark  & \ding{53} & \ding{53} \\
    P-FOLIO(\citeyear{p-folio2024})   &  1437 & \checkmark & \checkmark    & \ding{53} & \ding{53}\\
    \midrule
    \rowcolor[HTML]{E1EAFF}
    Our PC-FOL         & 2044 & \checkmark & \checkmark & \checkmark & \checkmark \\
    \bottomrule
    \end{tabular}
  \end{adjustbox}
  \vspace{-1.5em}
\end{table}

\begin{table*}[!h]
  \caption{Examples of linear-reasoning problem and proof-by-cases problem. The left side presents an example of linear-reasoning FOL question with a step-by-step proof sketch. The right side presents an example of proof-by-cases FOL question with a proof sketch.}
  \label{table_of_difference_between_linear_and_case}
  \centering
  \small
  \begin{adjustbox}{width=\linewidth}
    \begin{tabular}{c|cc}
    \toprule
    \makecell[l]{FOL \\ Type} & Linear-Reasoning & Proof-by-Cases \\
    \midrule
    Example  & 
    \makecell[l]{
    \textbf{NL Premises}\\
    1.No songs are visuals.\\
    2.All folk songs are songs.\\
    3.All videos are visuals.\\
    4.All movies are videos.\\
    5.All sci-fi movies are movies.\\
    6.Inception is a sci-fi movie.\\
    \textbf{NL Conclusions}\\
    Inception is a folk song.
    } 
    & 
    \makecell[l]{
    \textbf{NL Premises}\\
    1.Any person that is tall does not major in physics. \\
    2.All people who are not tall study quantum computing.\\
    3.All students major in math or physics. \\
    4.If a person majors in math, then the person studies algebraic geometry. \\
    5.If a person majors in math, then the person does not study quantum computing. \\
    6.Billy is a student who studies algebraic geometry. \\
    \textbf{NL Conclusions}\\
    Billy either majors in physics or studies quantum computing.
    } \\
    \midrule
    \makecell[l]{Proof \\ Process} & \includegraphics[height=10cm]{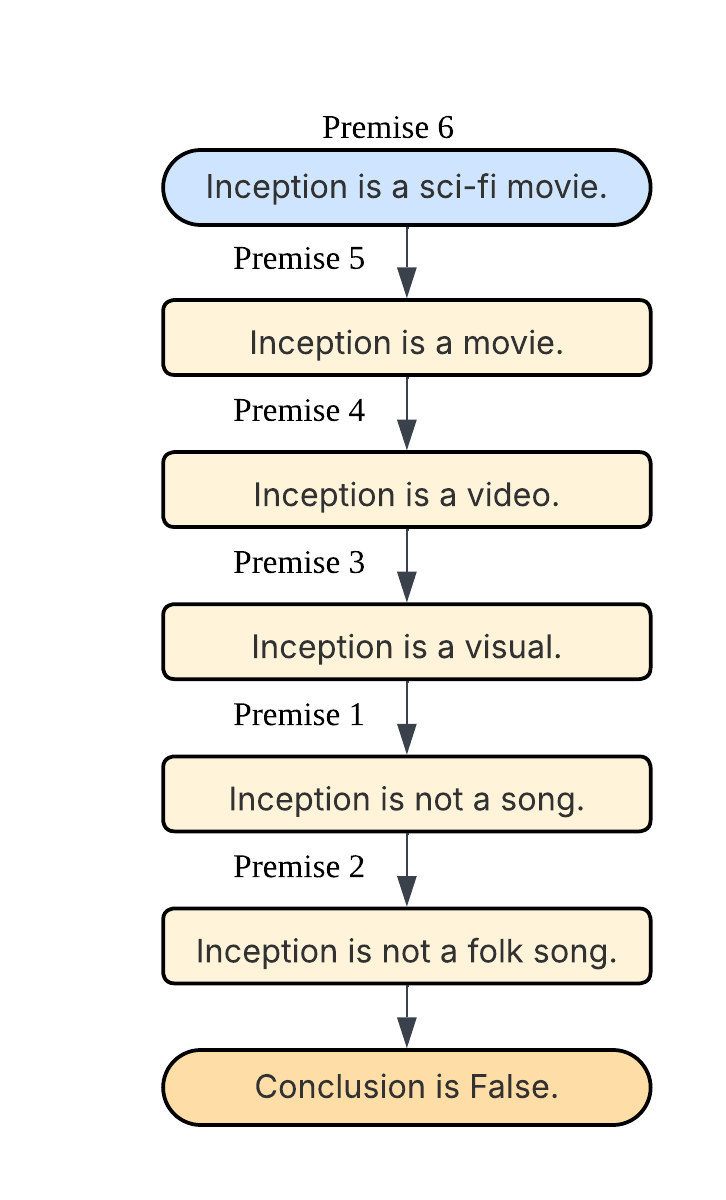}
    & \includegraphics[height=10cm]{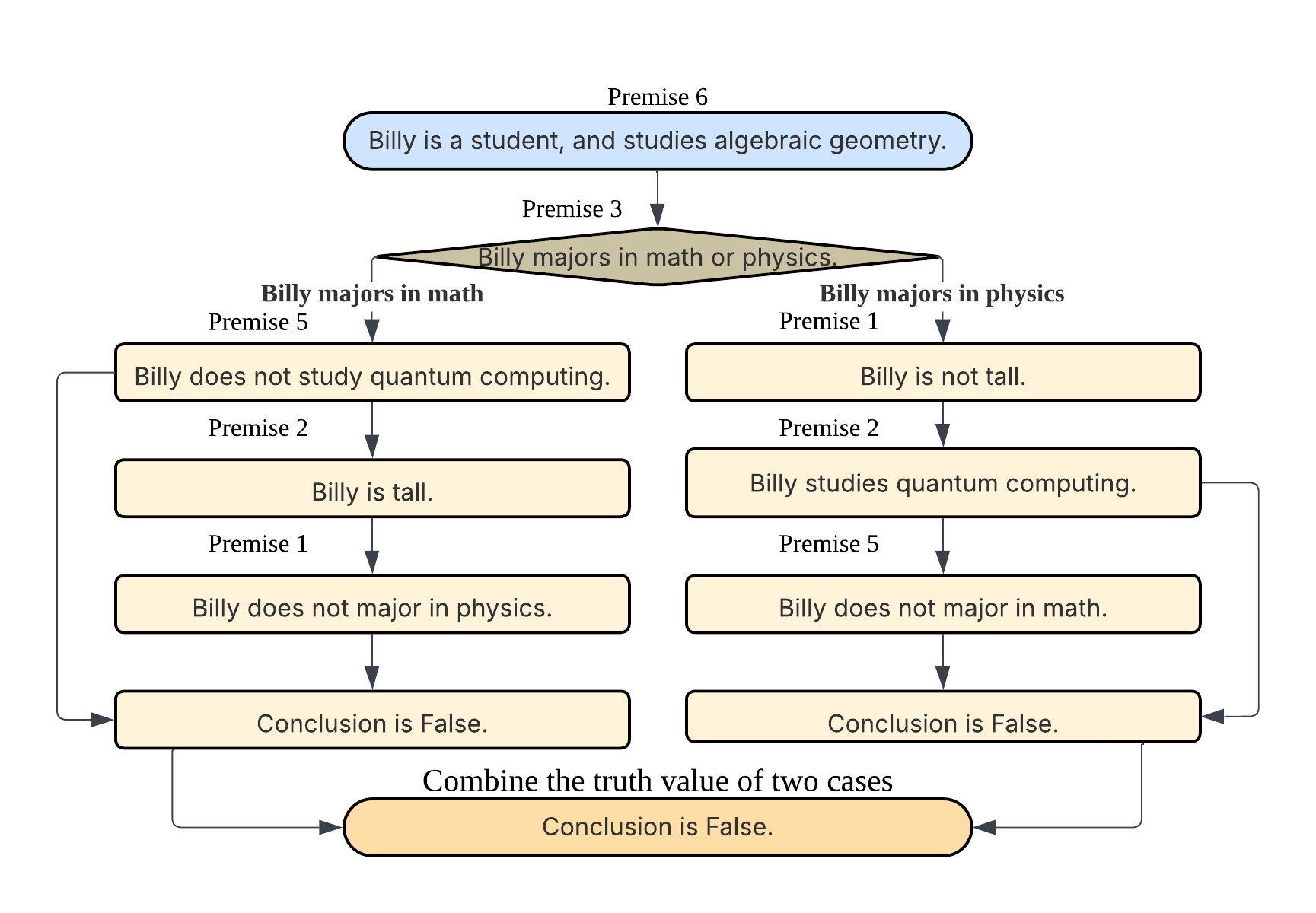}\\
    \bottomrule
    \end{tabular}
  \end{adjustbox}
\end{table*}

To address these challenges, we propose a \textbf{P}roof-by-\textbf{C}ases FOL benchmark (PC-FOL), which is a robust FOL reasoning dataset annotated by professional mathematicians and focuses on FOL problems that need to be solved by proof-by-cases technique. All instances in this dataset are equipped with a manually written natural language proof by our annotators. 
Note that we do not list proof-by-contradiction as an independent type, because all such proofs can be expressed in a standard process in the form of proof by cases. The details are shown in Appendix \ref{appendix_detail_proof_by_contradiction}.
To further assess the reasoning abilities of LLMs, inspired by CofQA \citep{wu2024cofcastepwisecounterfactualmultihop}, we apply \textbf{lexical substitution} by replacing certain nouns with random combinations of English alphabets in our instances. Therefore, LLMs cannot rely on their memories and must answer the problem based on the given premises.

Based on PC-FOL, we evaluate and report the reasoning ability of several leading LLMs~(Section \ref{section_experiments}). Experimental results show that \textbf{there is a huge gap between the performance of LLMs on linear-reasoning problems and proof-by-cases problems.} For example, GPT-4o performs well on our linear reasoning instances with 85$\%$ accuracy, but only gets a 51$\%$ accuracy on our proof-by-cases instances. 
Furthermore, we provide theoretical analysis~(Section \ref{section_theoretical_analysis}) by using the graphical model to explain why such a significant performance gap exists between the two types of reasoning problems.

\textbf{Our contributions are the following parts:} (1) \textbf{We manually curate a FOL dataset named PC-FOL}, which contains expert-annotated instances under two types of FOL questions for the first time. (2) \textbf{We benchmark the performance of the FOL reasoning task} for several LLMs by prompting them with zero-shot or few-shot examples, and find a substantial performance gap between the two types of FOL problems. (3) \textbf{We provide a plausible theory} to explain the reason why there exists such a substantial performance gap.

\section{Related Work}

\subsection{Dataset for FOL reasoning}

Compared with abundant mathematical reasoning datasets, the number of datasets specifically designed for FOL reasoning is limited. 
RuleTaker \citep{ruletaker2020} systematically evaluates a model's multi-step reasoning capabilities using data automatically generated from facts and rules.
LogicNLI \citep{logicnli2021} is a benchmark following basic principles of FOL to diagnose LLMs’ FOL reasoning ability.
ProntoQA \citep{prontoQA} creates the dataset through a four-stage process: ontology generation, proof construction, natural language conversion, and answer annotation. It is specifically designed to assess model performance in complex logical reasoning and planning tasks.
FOLIO \citep{folio2024} is the first reasoning dataset to combine natural language with FOL annotations. Based on FOLIO, P-FOLIO \citep{p-folio2024} gives the answers as step-by-step reasoning chains in the form of designed inference rules.

\subsection{LLM Natural Language Reasoning over Mathematical Problems}
Recent research has introduced various methods to enhance the natural language reasoning capabilities of LLMs in solving mathematical problems, while our work focuses on the obstacles that LLMs encounter in logical reasoning tasks. 
Here, we highlight several significant works that employ diverse approaches, including: (1) \textbf{Chain-of-Thought (CoT) Reasoning}, which significantly improves the ability of LLMs on complex problems. This strategy is employed by leading models such as OpenAI-O1 \citep{openai2024o1card} and DeepSeek-R1 \citep{deepseekai2025deepseekr1incentivizingreasoningcapability}.
(2) \textbf{Symbolic Reasoning.} 
MathCoder \citep{wang2024mathcoder} improves the performance by enabling LLMs to use code for modeling and deriving math equations.
AlphaGeometry \citep{AlphaGeometry2024} employs LLMs to translate natural language math problems into formal proof languages, demonstrating impressive generalization capabilities. Other models such as MathBERT \citep{peng2021mathbertpretrainedmodelmathematical} and MathGLM \citep{yang2023mathglm} optimize symbolic understanding by embedding mathematical formulas during pretraining, achieving outstanding results on algebraic and arithmetic tasks. 
(3) \textbf{Prompt Engineering}
The Thought Propagation method \citep{yu2023thoughtpropagation} explores designing a special prompt template, dynamically selecting the optimal path to enhance the flexibility and effectiveness of solving mathematical problems.
Active Prompting method \citep{diao-etal-2024-active} focuses on selecting high-information examples to optimize model responses.

\section{Preliminaries}
\label{section3.1_background}

\subsection{FOL Natural-Language Problem}

Propositional logic and FOL are two topics in the mathematical logic field, focusing on the study of formalized structures of reasoning. Propositional logic establishes reasoning relationships between atomic propositions using logical connectives such as implication ($\rightarrow$), conjunction ($\land$), disjunction ($\lor$), and negation ($\lnot$). FOL extends this framework by introducing quantifiers ($\forall$, $\exists$) and predicates, enabling the formal representation of individual objects and their attributes. 

Most widely used mathematical datasets, particularly FOL datasets, focus on step-by-step proofs. However, a substantial class of mathematical problems involves statements containing logical disjunctions (``or,'' $\lor$) or exclusive disjunctions (``either/or,'' $XOR$).
Such problems cannot be solved through straightforward sequential reasoning, since it is necessary to consider the distinct cases introduced by the disjunctive or exclusive disjunctive terms and determine whether the conclusion holds under each scenario.

Hence, in this work, we categorize natural language FOL problems into two types: \textbf{linear-reasoning} and \textbf{proof-by-cases}. 
The left side of Table \ref{table_of_difference_between_linear_and_case} presents an example of a linear-reasoning FOL question. 
The step-by-step answer for this question is given as follows: Based on the given premises, premise 6 establishes that ``Inception'' is a sci-fi movie. Subsequently, from premise 5, we know it is a movie; from premise 4, it is a video; and from premise 3, it is a visual. According to premise 1, ``Inception'' is not a song, and premise 2 states it is not a folk song. As a result, the conclusion ``Inception is a folk song'' is false.
The right side of Table \ref{table_of_difference_between_linear_and_case} provides an example of a proof-by-cases FOL question. To determine the truth value of the conclusion, one must discuss the two possible cases given by premise 3 separately. The reasoning proceeds as follows: In the left case, Billy majors in math, then Billy neither majors in physics (premises 5, 1, and 2) nor studies quantum computing (premise 5); In the right case, Billy majors in physics, then Billy must study quantum computing (premises 1, 2); Thus, in both cases, the conclusion is false; As a result, we can conclude that the truth value of the conclusion is false. 

Based on the above explanation, we can see that the two proof types are completely different. Therefore, it is essential to additionally evaluate the LLMs' reasoning abilities on proof-by-cases FOL problems and discover the performance gap of LLMs between the two types of FOL questions. 

\begin{table}[!t]
\caption{Basic statistics of PC-FOL.
$\#$Linear-Reasoning represents the number of linear-reasoning type instances, $\#$Proof-by-Cases represents the number of the proof-by-cases type instances.
  }
  \label{CMN_dataset_two_types}
  \centering
  \small
  \begin{tabular}{c|ccc}
    \toprule
    Dataset Name     & \makecell[l]{$\#$Linear\\Reasoning}  &  \makecell[l]{$\#$Proof-by\\-Cases} & $\#$Total \\
    \midrule
    PC-FOL & 511 & 511 & 1022\\
    PC-FOL-Replace & 511 & 511 & 1022\\
    \midrule
    PC-FOL Total & 1022 & 1022 & 2044 \\
    \bottomrule
  \end{tabular}
  \vspace{-1.0em}
\end{table}

\subsection{Graphical Model}

\begin{figure}[!h]
  \centering
  \includegraphics[width=\linewidth]{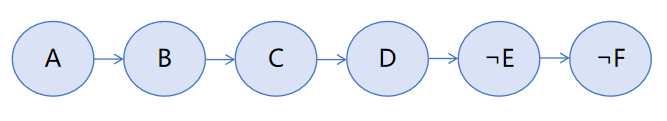}
  \caption{Abstracted reasoning chain for the left side example of Table \ref{table_of_difference_between_linear_and_case}.}
  \label{figure_graphical_model_linear_reasoning}
\end{figure}

When solving FOL problems in mathematical logic assignments, it is possible to abstract the logical reasoning process into a graphical representation, as shown in Figure \ref{figure_graphical_model_linear_reasoning}. In this graph, each property of “Inception” in the step-by-step proofs (corresponding to the left side of Table \ref{table_of_difference_between_linear_and_case}) is represented as a binary variable that takes the value true or false. For instance, let A denote ``Inception is a sci-fi movie,'' B denote ``Inception is a movie,'' and F denote ``Inception is a folk song.'' Arrows in the graph indicate the direction of logical inference.
Such representation is similar to the probabilistic graphical model. The probabilistic graphical model uses a graph to express the conditional dependencies among random variables. Mathematically, in a graphical model, if the events are $X_1, \ldots, X_n$, then their joint probability satisfies $P[X_{1},\ldots ,X_{n}]=\prod _{i=1}^{n}P[X_{i}|{\text{pa}}(X_{i})]$ where $\text{pa}(X_{i})$ is the set of parents of node $X_i$.


\section{The PC-FOL Dataset}

In this section, we detail the properties of our PC-FOL dataset, and present key statistics of our dataset in comparison to existing FOL datasets. The data collection process is outlined in Appendix \ref{data_collection_process}.

\textbf{Dataset Structure}
Our PC-FOL dataset contains 2044 instances categorized into two types: Linear-Reasoning type questions and Proof-by-Cases type questions. For each instance, natural language answers are carefully provided by by professional annotators. 
Furthermore, we apply the lexical substitution techniques by replacing the nouns with various nonsensical combinations of alphabetic characters on the PC-FOL dataset to construct the PC-FOL-Replace dataset.


\textbf{Number of Instances}
The basic statistics of PC-FOL are shown in Table \ref{CMN_dataset_two_types}. Since our dataset focuses on the proof-by-case type FOL reasoning problems, we balance the number of instances of the two types. Therefore, for each sub-dataset, our annotators collected 511 instances,
resulting in 1022 linear-reasoning instances and 1022 proof-by-case instances. 

\begin{table}[!t]
  \caption{Statistical comparison of PC-FOL with other datasets related to mathematical logic or FOL. $\#$Stories represents the number of distinct premise sets aligned to the instances. 
  $\#$Vocab represents the number of distinct words in the dataset. 
  }
  \label{CMN_dataset_statistics}
  \centering
  \begin{adjustbox}{width=\linewidth}
  \begin{tabular}{c|cccc}
    \toprule
    Dataset Name     & $\#$Stories  & $\#$Premises & $\#$Instances & $\#$Vocab \\
    \midrule
    RuleTaker(\citeyear{ruletaker2020}) & 456 & 11M & 500k & 65\\
    LogicNLI(\citeyear{logicnli2021}) & 2000 &  480k  & 20k & 1091\\
    FOLIO(\citeyear{folio2024})       &  413 &  6398 & 1204 & 4119 \\
    P-FOLIO(\citeyear{p-folio2024})   &  487 &  7620 & 1437 & 4658 \\
    \midrule
    PC-FOL & 291 & 5836 & 1022  & 3373 \\
    PC-FOL-Replace & 291 & 5870 & 1022 & 11293\\
    \midrule
    PC-FOL Total & 582 & 11706 & 2044 & 14666\\
    \bottomrule
  \end{tabular}
  \end{adjustbox}
   \vspace{-1.5em}
\end{table}

\begin{table*}[!h]
  \caption{Logical reasoning accuracy results (in $\%$) of zero-shot and few-shot prompting on PC-FOL dataset. ``Acc. Gap'' represents the performance gap between the Linear-Reasoning questions and Proof-by-Cases questions.}
  \label{table_few_shot_nl_reasoning_eng}
  \centering
  \small
  \begin{adjustbox}{width=\linewidth}
  \begin{tabular}{c|cc>{\columncolor{gray!20}}c|cc>{\columncolor{gray!20}}c}
    \toprule
    Model     & \makecell[l]{PC-FOL \\ Linear-Reasoning}  & \makecell[l]{PC-FOL \\ Proof-by-Cases} &\makecell[l]{PC-FOL \\ Acc. Gap} & \makecell[l]{PC-FOL-Replace \\ Linear-Reasoning} & \makecell[l]{PC-FOL-Replace \\ Proof-by-Cases} & \makecell[l]{PC-FOL-Replace \\ Acc. Gap}\\
    \midrule
    GPT-4o \textsubscript{\textbf{0-shot}} & 85.13 & 51.08 & \textbf{34.05}{\small ↓} & 84.34 & 51.66 & \textbf{32.68}{\small ↓}\\
    GPT-4o \textsubscript{\textbf{3-shot}} & 83.17 & 55.58 & \textbf{27.59}{\small ↓} & 82.39 & 52.64 & \textbf{29.75}{\small ↓}\\
    \midrule
    GPT-4.1 \textsubscript{\textbf{0-shot}} & 89.24 & 69.86 & \textbf{19.38}{\small ↓} & 88.45 & 63.80 & \textbf{24.65}{\small ↓}\\
    GPT-4.1 \textsubscript{\textbf{3-shot}} & 86.30 & 71.23 & \textbf{15.07}{\small ↓} & 84.34 & 68.69 & \textbf{15.65}{\small ↓}\\
    \midrule
    o4-mini \textsubscript{\textbf{0-shot}} & 90.80 & 79.84 & \textbf{10.96}{\small ↓} & 88.26 & 77.89 & \textbf{10.37}{\small ↓}\\
    o4-mini \textsubscript{\textbf{3-shot}} & 90.02 & 80.43 & \textbf{9.59}{\small ↓} & 88.45 & 79.45 & \textbf{9.00}{\small ↓}\\
    \midrule
    Llama 3-70B \textsubscript{\textbf{0-shot}} & 80.63 & 46.77 & \textbf{33.86}{\small ↓} & 79.84 & 50.88 & \textbf{28.96}{\small ↓} \\
    Llama 3-70B \textsubscript{\textbf{3-shot}} & 83.56 & 54.21 & \textbf{29.35}{\small ↓} & 80.63 & 53.62 & \textbf{27.01}{\small ↓} \\
    \midrule
    Deepseek-V3 \textsubscript{\textbf{0-shot}} & 89.63 & 76.71 & \textbf{12.92}{\small ↓} & 86.69 & 73.39 & \textbf{13.30}{\small ↓} \\
    Deepseek-V3 \textsubscript{\textbf{3-shot}} & 90.02 & 77.10 & \textbf{12.92}{\small ↓} & 85.71 & 72.02 & \textbf{13.69}{\small ↓} \\
    \midrule
    Qwen3-14B \textsubscript{\textbf{0-shot}} & 87.28 & 63.60 & \textbf{23.68}{\small ↓} & 84.54 & 59.10 & \textbf{25.44}{\small ↓} \\
    Qwen3-14B \textsubscript{\textbf{3-shot}} & 89.04 & 66.34 & \textbf{22.70}{\small ↓} & 86.89 & 64.97 & \textbf{21.92}{\small ↓} \\
    \bottomrule
  \end{tabular}
  \end{adjustbox}
\end{table*}

\textbf{Stories and Premises} 
Table \ref{CMN_dataset_statistics} shows that our dataset contains 2044 instances and 11706 premises. Using the same definition in FOLIO, we call the premise sets in an instance a story. By this definition, our PC-FOL dataset contains 291 distinct stories.

\textbf{Vocabulary}
Our PC-FOL dataset contains a vocabulary of 3373 unique words. By replacing noun words with combinations of random English alphabets and making minor modifications to the instances in PC-FOL, we construct the PC-FOL-Replace dataset, significantly increasing the vocabulary size to 11293. 





\section{Experiments}\label{section_experiments}
We conduct several experiments based on the proposed PC-FOL dataset, aiming to answer the following research questions:
1) Is there a performance gap for LLMs between the two types of FOL problems? 2) Are the proofs generated by LLMs correct? 

\subsection{Experiment Setting}\label{section_proof_evaluation}
\label{section_task}

\textbf{Experimental Tasks} Based on PC-FOL dataset, we conduct two types of experiments: Natural Language FOL Reasoning and Proof Evaluation. 
The Natural Language FOL Reasoning task aims to evaluate the reasoning capabilities of selected LLMs, where the LLMs are given the premises and are asked to give the truth value of the conclusion from the following three choices: ``True'', ``False'', and ``Unknown''. 
The prompt templates used in this task are provided in Appendix \ref{prompt_of_experiments}.
The Proof Evaluation task is designed to evaluate the correctness of the proofs generated by LLMs. For each instance, the premises are provided, and the LLMs are tasked with generating natural language proofs along with the corresponding truth value of the conclusion. We also conduct additional fine-tuning experiments on the PC-FOL dataset, and the corresponding details are presented in Appendix \ref{appendidx_fine_tuning_results}.

\textbf{Metrics} To evaluate the performance of the LLMs, we use Accuracy as the metric to evaluate the generated truth values. Also, following the P-FOLIO \citep{p-folio2024} paper, we try two metrics to evaluate the generated proofs: ROUGE \citep{lin-2004-rouge} and pass$@$k \citep{Chen2021EvaluatingLLpassk}. 
The \textbf{Accuracy} metric reports the percentage of correct truth value (True/False/Unknown) generated by the tested LLMs. 
The \textbf{ROUGE} metrics including ROUGE-1, ROUGE-2, and ROUGE-L \citep{lin-2004-rouge}, are used to compare the model-generated proofs with human-written proofs. 
The \textbf{pass$@$k} metric is defined as the same in P-FOLIO \citep{p-folio2024}: After sampling $k$ proofs from the tested LLM, the pass$@$k represents the percentage of instances in which at least one generated proof follows the same reasoning process as the annotated proof. The verification process is automatically checked by GPT-4o, and the prompt templates are provided in Appendix \ref{prompt_of_experiments}.


\textbf{Models} We employ both proprietary LLMs and open-source LLMs in experiments. The proprietary LLMs include GPT-4o \citep{openai2024gpt4ocard}, GPT-4.1 \citep{openai2024gpt4-1}, o4-mini \citep{openai2025o4mini}. The open-source LLMs tested in our experiments are Llama-3.3-70B \citep{grattafiori2024llama3herdmodels}, deepseek \citep{deepseekai2025deepseekv3technicalreport} and Qwen3 \citep{yang2025qwen3technicalreport}. More details are deferred to Appendix \ref{appendix_detail_of_experiment}.

\begin{table*}[!ht]
  \caption{Result of ROUGE metrics for few-shot prompting with GPT-4o on PC-FOL dataset.}
  \label{rouge_metric_many_shot_reasoning}
  \centering
  \small
  \begin{tabular}{c|cccccccc}
    \toprule
    $\#$Shot     & \makecell[l]{Linear \\ R-1} & \makecell[l]{Linear\\ R-2} & \makecell[l]{Linear \\ R-L} & \makecell[l]{Linear \\ Acc} & \makecell[l]{Case \\ R-1} & \makecell[l]{Case \\ R-2} & \makecell[l]{Case \\ R-L} & \makecell[l]{Case \\ Acc} \\
    \midrule
    0-shot & 43.56  & 29.24  & 31.67 & 85.13 & 59.05 & 36.73 & 34.37 & 51.66\\
    5-shot & 54.88 & 37.60 & 42.42 & 85.77 & 62.00 & 39.99 & 37.97 & 53.95 \\
    10-shot & \textbf{55.98} & 38.50 & \textbf{43.23} & 86.36 & 62.16 & 40.36 & 38.57 & 55.73\\
    20-shot & 55.77 & 38.58 & 43.02 & 87.15 & 62.86 & 40.96 & 39.27 & 55.53 \\
    40-shot & 55.96 & \textbf{38.74} & 43.11 & \textbf{87.94} & \textbf{63.38} & \textbf{41.20} & \textbf{39.76} & \textbf{57.91}\\
    \bottomrule
  \end{tabular}
  \vspace{-1em}
\end{table*}

\subsection{Natural language FOL reasoning}\label{section_5.1_nl_FOL_reasoning}

We present the results of evaluating the natural language FOL reasoning capabilities of various LLMs based on PC-FOL dataset in Tables \ref{table_few_shot_nl_reasoning_eng}. More experimental results can be found in Appendix \ref{Appendix_more_experiments}.

Based on the results, we can observe that \textbf{there are significant performance gaps between the linear-reasoning FOL problems and the proof-by-cases FOL problems across all experiments.} For the GPT-4o model, while it achieves 85.13$\%$ accuracy on the 0-shot task for the PC-FOL linear-reasoning dataset, its accuracy on the proof-by-cases dataset is 34.05$\%$ lower. A similar trend is observed in the 3-shot task, where the accuracy for the linear-reasoning problems is 83.17$\%$, compared to 55.58$\%$ for the proof-by-cases problems.
The DeepSeek-v3 model performs best overall but still exhibits a 12.92$\%$ gap, with 89.63$\%$ accuracy on linear-reasoning problems and 76.71$\%$ on proof-by-cases problems.
On the PC-FOL-Replace dataset, we also observe a 1-5$\%$ accuracy gap compared to the PC-FOL dataset for each tested LLM. 

\subsection{Proof Evaluation}\label{section_5_3_proof_evaluation}

\textbf{Pass@k Evaluation}

\begin{table}[!h]
  \caption{Pass$@$k results for GPT-4o model on the PC-FOL dataset.}
  \label{passk_gpt4o_reasoning}
  \centering
  \small
  \begin{tabular}{c|c|cc}
    \toprule
    Model  & \makecell k   & \makecell[l]{PC-FOL \\ Linear-Reasoning} & \makecell[l]{PC-FOL \\ Proof-by-Cases} \\ 
    \midrule
    \multirow{5}{*}{GPT-4o}
        & 1 & 82.58 & 46.38\\
        & 2 & 96.67 & 71.82\\
        & 3 & 99.41 & 85.13\\
        & 4 & 100.0 & 92.95\\
        & 5 & 100.0 & 95.50\\
    \bottomrule
  \end{tabular}
\end{table}

Table \ref{passk_gpt4o_reasoning} presents the Pass@k results for GPT-4o on the PC-FOL dataset. The Pass@k metric represents the percentage of instances where at least one of $k$ sampled proofs is deemed to match the expert-written proof by the evaluation LLM model. \\
The results indicate that \textbf{as $k$ increases, the Pass@k metric improves significantly.}
For example, on the linear-reasoning type problems, the Pass@k score reaches 100$\%$ when $k=4$, indicating that the evaluation model believes there is at least one correct proof for every instance. For proof-by-case type problems, when $k=1$, only 46.38$\%$ of instances are considered to get a correct answer from GPT-4o, and this percentage rises to 95.50$\%$ when $k=5$, showcasing a substantial metric increase.

\textbf{Few-shot prompting}
Table~\ref{rouge_metric_many_shot_reasoning} indicates that increasing the number of proof examples leads to improved performance. Specifically, the ROUGE scores improve by approximately 10$\%$ for linear reasoning problems and about 5$\%$ for case-based reasoning problems when the number of shots increases from 0 to 40. Besides, label accuracy shows a relative improvement of 5$\%$ in the 40-shot setting compared to the 0-shot baseline.

\subsection{Manually Checking Proofs}

In Section \ref{section_5.1_nl_FOL_reasoning}, we observe that \textbf{LLMs exhibit low accuracy when solving proof-by-cases type FOL problems.} Consequently, employing another LLM to evaluate the correctness of the answers (e.g., the Pass@k metric in Section \ref{section_5_3_proof_evaluation}), would likely introduce significant errors. To mitigate this issue, we further conduct a manual experiment to evaluate the answers generated by the GPT-4o model. Specifically, 
in this experiment, we utilize GPT-4o (web interactive version) to generate a proof with a corresponding label for each problem in our PC-FOL dataset. These proofs are then examined by a professional mathematician, who categorized each instance into one of three categories: Wrong Label with Wrong Proof, Correct Label with Wrong Proof, and Correct Label with Correct Proof. The distribution of results across these categories is reported in Table \ref{manully_check_result}. 

\begin{table}[!h]
  \caption{Proportions of the three answer categories (round to two decimal places) of GPT-4o (web interactive) on our PC-FOL dataset.}
  \label{manully_check_result}
  \centering
  \small
  \begin{tabular}{c|cc}
    \toprule
    Category    & \makecell[l]{PC-FOL \\ Linear-Reasoning} & \makecell[l]{PC-FOL \\ Proof-by-Cases} \\ 
    \midrule
      \makecell[l]{Wrong Label \\ Wrong Proof}  & 15.07$\%$ & 45.79$\%$ \\
    \midrule
       \makecell[l]{Correct Label \\ Wrong Proof} & 4.31$\%$ & 28.18$\%$ \\
    \midrule
       \makecell[l]{Correct Label \\ Correct Proof} & 80.63$\%$ & 26.03$\%$ \\
    \bottomrule
  \end{tabular}
  \vspace{-1em}
\end{table}

We can now answer the research question ``(2): Are the proofs generated by LLMs correct?'' based on the results in Table \ref{manully_check_result}, which indicates that for linear-reasoning problems, \textbf{when the model assigns a correct label, the corresponding proof is also likely to be correct.} However, for proof-by-cases problems, not only is the label accuracy only \textbf{54.21$\%$}, but \textbf{fewer than half} of these samples with the correct label actually get a correct proof. Through our manual checking of the proofs, we identify \textbf{four main reasons for the proof errors}: misapplication of premises, misinterpretation of disjunctive statements, mistakes in inductive and deductive reasoning, and semantic misunderstandings. In particular, for proof-by-cases problems, the predominant cause of incorrect proof is the \textbf{misapplication of disjunctive statements}, where the tested LLM often misunderstands exclusive disjunctions, or fails to provide proofs that address separate scenarios.



\subsection{Main Findings}

Below we summarize our main findings, which are the answers of the research question (1).
 
\textbf{For all tested LLMs, there is a huge performance gap between Linear-Reasoning type instances and Proof-by-Cases type instances.} The average accuracy performance gap between the two types of questions for different LLMs, as shown in Tables \ref{table_few_shot_nl_reasoning_eng}, is calculated as follows: GPT-4o (31.02$\%$), GPT-4.1 (18.69$\%$), o3-mini (9.98$\%$), Llama-3 (29.80$\%$), Deepseek-V3 (13.21$\%$), Qwen3-14B (23.44$\%$). Through manual evaluation and considering the distribution of the number of premises (shown in Appendix \ref{distribution_of_premises}), we found no evidence to suggest that this performance gap is related to the number of premises. 
Instead, the performance gap is attributed to the fundamental problem-solving methods required for linear-reasoning problems and proof-by-cases problems.

\textbf{Lexical substitution has only a marginal effect on model performance.} The average performance difference between the PC-FOL and PC-FOL-Replace for different LLMs is calculated as follows: GPT-4o (0.98$\%$), GPT-4.1 (2.84$\%$), o3-mini (1.76$\%$), Llama-3 (0.05$\%$), Deepseek-V3 (3.91$\%$), Qwen3-14B (2.69$\%$). The only notable exception appears in Table \ref{table_few_shot_nl_reasoning_eng}, where Llama-3 shows an abnormal accuracy gap on the 0-shot proof-by-cases task. In general, the performance degradation caused by lexical substitution can be attributed to the inability of LLMs to correctly recognize the substituted nouns. This difficulty arises because nouns in the PC-FOL-Replace dataset are randomly generated alphabetic strings, which are unlikely to have appeared in the training corpora of any LLM.


\section{Theoretical Analysis}
\label{section_theoretical_analysis}
To explain the substantial performance gap exhibited between linear-reasoning problems and proof-by-cases problems, we employ the \textbf{probabilistic graphical model} to discuss theoretical approaches for analyzing the performance of LLMs on mathematical logic tasks, propose specific theoretical assumptions, and introduce a variant of the probabilistic graphical model tailored to mathematical logic problems.

\subsection{Reasoning Ability of LLMs for Linear-Reasoning problem}

On the one hand, directly using the graphical model (described in Section~\ref{section3.1_background}) can hardly define the linear reasoning ability of LLM. On the other hand, traditional analysis of neural-based NLP models typically defines the probability distribution as $P[X_{1},\ldots,X_{n}]=\prod _{i=1}^{n}P[X_{i}|Context, X_1, \ldots, X_{i-1}]$, where $X_i$ represents the $i$-th token of the output sequence, 
which offers limited explanatory power regarding the mathematical principles or underlying mechanisms of black-box neural models. To address this limitation, we propose a variant graphical model constructed upon the following assumptions and definitions:
\begin{enumerate}
    \item Define each $X_i$ as a proposition described in the given premises.
    \item The conditional probability $P[X_{i}|Context, X_1, \ldots, X_{i-1}]$ can be viewed as selecting the next proposition based on an augmented context.
    \item Fixed the LLM as a function $F$. Define $p_F(Augmented \ Context)$ as the correctness percentage of selecting the correct proposition from context. The values for different augmented contexts are assumed to be the same, defined as $p_F$.
\end{enumerate}

Therefore, the generated proof with the process $X = \{X_1, \ldots, X_n\}$ has the form of $P[X_{1},\ldots,X_{n}]=\prod _{i=1}^{n}P[X_{i}|Augment \ Context]$, and the probability of generating the correct reasoning chain is $p_F^k$, if there are $k$ steps in the corrected proof.

Based on the previous assumptions, we can prove the following theorem:
\begin{theorem}\label{theorem_6_1}
    Under previous assumptions, for a LLM $F$, a linear-reasoning dataset $D$ with the distribution $P_{D}$ of reasoning steps, the probability that the LLM $F$ can give a correct proof is $\sum\limits_{k=1}^{\infty} p_F^k \cdot P_D(|X|=k)$, or $E_{X \sim P_D}[p_F^{|X|}].$
\end{theorem}

The justification for the assumptions, together with the complete proof of Theorem~\ref{theorem_6_1}, are provided in Appendix~\ref{appendix_proof}.
As a result, by using the distribution of the proof steps in the linear-reasoning part (Appendix \ref{distribution_of_proof_steps}), we can then calculate the probability $p_F$, the accuracy for getting the next step for different LLM $F$. For example, since Table \ref{manully_check_result} shows that the correct proof rate for GPT-4o over our dataset is 80.63$\%$, by using Theorem \ref{theorem_6_1}, we can calculate that $p_{GPT-4o}=92.62\%$.

\subsection{Reasoning Ability of LLMs for Proof-by-Cases problem}

We have shown the abstract reasoning chain for a linear-reasoning example in Figure \ref{figure_graphical_model_linear_reasoning}. However, for proof-by-cases questions, the abstract reasoning chain is completely different. The upper part of Figure \ref{figure_graphical_model_proof_by_case} shows an abstracted reasoning chain for the right side example of Table \ref{table_of_difference_between_linear_and_case}, and the lower part shows an abstracted reasoning chain for an example that exists a contradiction in one case. The original description of the lower part example is deferred to Appendix \ref{appendix_proof_proofbycases}.

\begin{figure}[t]
  \centering
  \includegraphics[width=.45\textwidth]{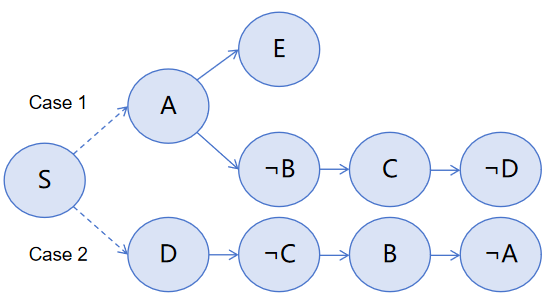}
  \hspace{1cm}
  \includegraphics[width=.45\textwidth]{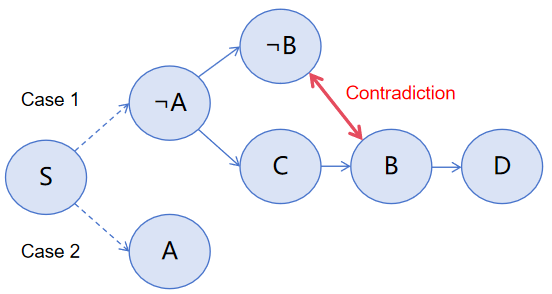}
  \caption{\textbf{Upper}: The abstracted reasoning chain for right side example of Table \ref{table_of_difference_between_linear_and_case}. \textbf{Lower}: The abstracted reasoning chain for a typical example that existing a contradiction in a subcase. The dotted line represents a possible case inferred from a certain premise and the previous steps, and the red double-headed arrow denotes the existence of a contradiction.}
  \label{figure_graphical_model_proof_by_case}
  \vspace{-1em}
\end{figure}

In mathematical logic problem-solving, a standard proof for a proof-by-cases problem must comprise two parts: 1) Complete logical chains for all cases where a contradiction arises, demonstrating that these cases are invalidated by the contradiction; and 2) for all cases that satisfy the given requisite (e.g., ``If A holds, then xxxxxx.''), a corresponding logical deduction that verifies the validity of the conclusion.

Thus, the proof of proof-by-cases type problems should be defined as a set of reasoning chain sets. Mathematically, suppose the set $E$ contains the cases that are possible to exist based on the given premises, and the set $NE$ contains the cases that can not exist (there is a contradiction in them) based on the given premises. Assuming that, for the $i$-th possible case, the reasoning chain for the given question is represented as $\{X^E_{i,1}, \cdots, X^E_{i,n_i}\}$, and for the $j$-th impossible case, the reasoning chain for the given question is represented as $\{X^{NE}_{j,1}, \cdots, X^{NE}_{j,n_j}\}$. Under these definitions, the proof should be the set $\{\{X^{NE}_{j,1}, \cdots, X^{NE}_{j,n_j}\} \forall case \ j\in NE, \{X^{E}_{k,1}, \cdots, X^{E}_{k,n_i}\} for \ selected \  case \ k\in E\}$, where the $k$ means that the $k$-th case of the set $E$ satisfies the given requisite of the question.

Therefore, based on the previous assumptions, we can prove the following theorem:
\begin{theorem}\label{theorem_6_2}
    Under previous assumptions, for a LLM $F$, a dataset $D$ with the distribution $P_{D}$ of the set of reasoning chains. Define $p_{F,cases}(X)$ as the probability that the LLM can correctly identify all the cases in $X$ required for discussion, then the probability that the LLM $F$ can give a correct proof is $\sum p_{F,cases}(X)\cdot p_F^{\sum_{i} |X^{NE}_i|+|X^{NE}|+\sum\limits_{selected \ k}|X^{E}_k|} \cdot P_D(X=\{all \{X^{NE}\}, selected \{X^{E}\}\})$
\end{theorem}

If people define a correct proof as the set of all the reasoning steps in different cases, even for the cases that do not satisfy the requisite of the question, then the Theorem \ref{theorem_6_2} will be simplified as
\begin{theorem}\label{theorem_6_3}
    Under previous assumptions, for a LLM $F$, a dataset $D$ with the distribution $P_{D}$ of the set of reasoning chains. Then the probability that the LLM $F$ can give a correct proof is $\sum\limits_{k=1}^{\infty} p_F^k \cdot P_D(|X|+|X^{NE}|=k)$, or $E_{X \sim P_D}[p_F^{|X|+|X^{NE}|}].$
\end{theorem}


From Theorems \ref{theorem_6_2} and \ref{theorem_6_3}, we can identify \textbf{two primary factors leading to LLMs' errors} in proof-by-cases problems: \textbf{incorrect selection of cases} (reflected by a low value of $p_{F,cases}$) and \textbf{an excessive number of steps} required to show the logical chains for each valid case (reflected by a large exponent $k$ in $p_F^k$). These theoretical findings align with the issues observed in our experimental results, demonstrating the reasonableness of our theoretical framework for the proof-by-cases problems.

\section{Limitation and Conclusion}
In this paper, we introduced the PC-FOL dataset, a novel dataset for evaluating the FOL reasoning capabilities of LLMs, which contains expert-annotated instances under two types of FOL questions: linear-reasoning and proof-by-cases. Although the scale of this dataset is limited compared to the extensive corpora typically used for LLM pretraining, the dataset can still be used for LLM evaluation.

Extensive experiments involving  widely-used LLMs revealed that while LLMs perform relatively well on traditional linear reasoning FOL tasks, they struggle significantly with problems requiring proof-by-cases
technique. 

Finally, to explain this phenomenon, we proposed specific theoretical assumptions and designed a variant probabilistic graphical model specifically for FOL problems. Note that, in our assumptions, the probability of a correct single-step inference is set as constant, which simplifies the inference process of complex LLM models.

This work highlighted the limitations of current LLMs and underscores the need for robust methods that can handle a broader spectrum of real-world problem-solving, particularly case-based reasoning.



\newpage
\bibliography{bibliography_llm_mathlogic_proof}
\bibliographystyle{plainnat}

\section*{Checklist}



\begin{enumerate}

  \item For all models and algorithms presented, check if you include:
  \begin{enumerate}
    \item A clear description of the mathematical setting, assumptions, algorithm, and/or model. [Yes]
    \item An analysis of the properties and complexity (time, space, sample size) of any algorithm. [Not Applicable]
    \item (Optional) Anonymized source code, with specification of all dependencies, including external libraries. [Yes]
  \end{enumerate}

  \item For any theoretical claim, check if you include:
  \begin{enumerate}
    \item Statements of the full set of assumptions of all theoretical results. [Yes]
    See Section \ref{section_theoretical_analysis}.
    \item Complete proofs of all theoretical results. [Yes]
    See Appendix \ref{appendix_proof}.
    \item Clear explanations of any assumptions. [Yes]     
    See Appendix \ref{appendix_proof}.
  \end{enumerate}

  \item For all figures and tables that present empirical results, check if you include:
  \begin{enumerate}
    \item The code, data, and instructions needed to reproduce the main experimental results (either in the supplemental material or as a URL). [Yes] See Appendix \ref{anonymous_link_of_dataset}.
    \item All the training details (e.g., data splits, hyperparameters, how they were chosen). [Yes]
    \item A clear definition of the specific measure or statistics and error bars (e.g., with respect to the random seed after running experiments multiple times). [Yes]
    \item A description of the computing infrastructure used. (e.g., type of GPUs, internal cluster, or cloud provider). [Yes]
  \end{enumerate}

  \item If you are using existing assets (e.g., code, data, models) or curating/releasing new assets, check if you include:
  \begin{enumerate}
    \item Citations of the creator If your work uses existing assets. [Yes]
    \item The license information of the assets, if applicable. [Yes]
    \item New assets either in the supplemental material or as a URL, if applicable. [Yes]
    \item Information about consent from data providers/curators. [Not Applicable]
    \item Discussion of sensible content if applicable, e.g., personally identifiable information or offensive content. [Not Applicable]
  \end{enumerate}

  \item If you used crowdsourcing or conducted research with human subjects, check if you include:
  \begin{enumerate}
    \item The full text of instructions given to participants and screenshots. [Not Applicable]
    \item Descriptions of potential participant risks, with links to Institutional Review Board (IRB) approvals if applicable. [Not Applicable]
    \item The estimated hourly wage paid to participants and the total amount spent on participant compensation. [Not Applicable]
  \end{enumerate}

\end{enumerate}

\clearpage
\appendix
\thispagestyle{empty}

\onecolumn
\aistatstitle{Paper Submission to AISTATS 2026: \\
Supplementary Materials}


\section{Distribution of number of Premises}\label{distribution_of_premises}


The distribution of the number of premises in each instance of the PC-FOL dataset is shown in Figure \ref{figure_distribution_premises}. The other part, PC-FOL-Replace, has almost the same distribution as the PC-FOL dataset. Therefore, we do not list the same image multiple times.

\begin{figure*}[!h]
  \centering
  \includegraphics[width=\linewidth]{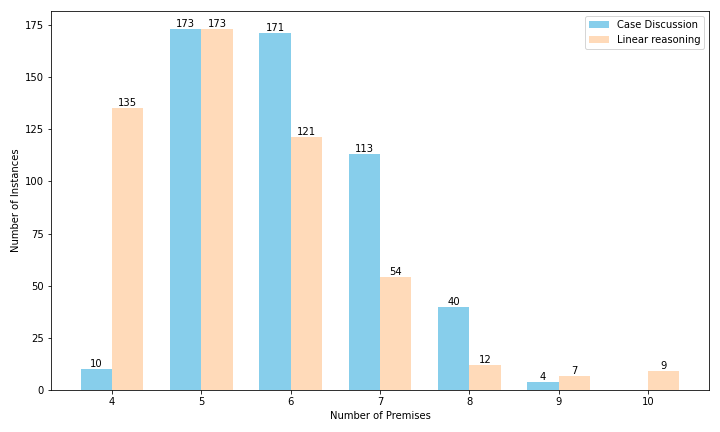}
  \caption{The distribution of the number of premises in each instance of the PC-FOL dataset. Blue color represents the distribution of the proof-by-case type problems, and the yellow color represents the distribution of the linear reasoning type problems.}
  \label{figure_distribution_premises}
\end{figure*}

\section{Dataset Description}

\subsection{Data collection process}\label{data_collection_process}

Most instances in our PC-FOL dataset come from the instances in the FOLIO \citep{folio2024} dataset with little modifications, or from the exercises or homework in some mathematical logic courses. \\
Through carefully and manually checking each instance in FOLIO, we utilize the data that is logically sound and has the correct label. We re-write some instances which have logical issues, and re-label those data with the wrong label. 
We also utilize the exercises or homework in mathematical logic courses. By modifying the exercises in some lecture notes, or re-writing the homework from FOL formulas to natural language, many instances in the proof-by-case category have been collected in our dataset. \\
To make the logical reasoning complicated enough, the number of premises in each instance is at least 4. Unlike the linear-reasoning problems shown in previous work, the reasoning depth of proof-by-cases problems can not be defined well and the reason is given in Appendix \ref{reason_of_depth_not_well}. Therefore, we show the distribution of the number of premises in each instance instead. The figure of the distribution is shown in Appendix \ref{distribution_of_premises}. \\
For each instance, we provide a natural language answer, and the answers are all written by our professional annotators. \\
The quality control process is shown in Appendix \ref{appendix_quality_control}.

\subsection{Quality control of PC-FOL dataset}\label{appendix_quality_control}

\textbf{Dataset Annotation} 
(1) To ensure that our dataset is annotated with high precision and professionalism, our annotators are all professional mathematicians with a math PhD degree and have taken at least one graduate-level course related to mathematical logic or discrete math. 
(2) Our annotators wrote the premises and the conclusions by using the same format as the FOLIO dataset. 
(3) Our annotators are native English speakers, or have the ability to write professional academic English.

\textbf{Natural Language Quality}
Since our annotators are professional mathematicians, we can make sure that the sentences in our dataset, especially the natural language answers of the instances, have the same writing style as the answers of exercises from graduate-level mathematics textbooks.
Besides, all the sentences are checked with two grammar checking tools: a traditional software called Grammarly, and a LLM-based software called Writefull.

\textbf{Cross checking}
Each instance and its answer are double-checked by two annotators. 

\subsection{Anonymous Link}\label{anonymous_link_of_dataset}

We provide an anonymous link for downloading our PC-FOL dataset: \url{https://www.kaggle.com/datasets/c425f3c00b279f2b843a71456b9ecc9ddba8c0d12b9f83deb6bd6b0435db85c6}.

\section{Explanation about the reasoning depth}\label{reason_of_depth_not_well}

In this section, we give an example to show why we believe that the reasoning depth of proof-by-cases FOL question can not be defined well.

Suppose that we have such premises
\begin{itemize}
    \item If C holds, then either A or B holds.
    \item If C holds and B holds, then A holds.
    \item If C holds and B does not hold, then A does not hold.
\end{itemize}

And a statement “C holds”.

To evaluate whether the statement is true or false, the standard way is to consider two cases:

Case 1: C holds.

Case 2: C does not hold.

In case 1, since C holds, consider two subcases:

Case 1.1 A holds. Then by premise 1, B does not hold. By premise 3, A does not hold, which makes a contradiction.

Case 1.2 A does not hold. Then by premise 1, B holds. By premise 2, A holds, which makes a contradiction too.

Since both subcases are impossible, case 1 is impossible.

When attempting to construct a reasoning chain for Case 1, one may observe the following process:

(A holds)

$\rightarrow$ premise 1, (B does not hold)

$\rightarrow$ premise 3, (A does not hold)

$\rightarrow$ premise 1, (B holds)

$\rightarrow$ premise 2, (A holds)

$\rightarrow$ ... (cycle continues)

As a result, we obtain a cyclic graph, rather than a typical linear chain of standard reasoning questions. Because the length of the longest path in a cyclic graph is not well-defined (especially in graph theory), it is unusual to define the reasoning depth for such FOL questions. Previous work ignored this situation and they may give the depth as the depth in one case.

\section{Prompt used in Experiments}\label{prompt_of_experiments}

The prompts used in our experiments are shown below. The templates are designed from the ideas of several published FOL reasoning papers. 

\subsection{Prompts used for generating proofs and labels}

Reasoning with proof version:

\fbox{%
  \begin{minipage}{0.95\linewidth}
Using deductive reasoning, find out the truth values of the conclusions based on the premises. The truth value can be True, False or Uncertain. First show the reasoning process, and then output the truth value in the format of "Truth value: ".\\
Premises: \texttt{nl\_premises}\\
Conclusion: \texttt{conclusion}
  \end{minipage}
}

Multi-shot version:

\fbox{%
  \begin{minipage}{0.95\linewidth}
Here are some examples of deductive reasoning problems and their correct truth values:\\
\texttt{few\_shot\_examples}\\
Now, using deductive reasoning, find out the truth values of the conclusions based on the premises. The truth value can be True, False or Uncertain. First show the reasoning process, and then output the truth value in the format of "Truth value: ".\\
Premises: \texttt{nl\_premises}\\
Conclusion: \texttt{conclusion}
  \end{minipage}
}

Chinese version:

\fbox{%
  \begin{minipage}{0.95\linewidth}
\begin{CJK*}{UTF8}{gbsn}
使用演绎推理，找出结论的真值。真值可以是“真”、“假”或“不确定”。首先展示推理过程，然后以“真值: ”的格式输出真值。\\
前提: \texttt{nl\_premises}\\
结论: \texttt{conclusion}
\end{CJK*}
  \end{minipage}
}

Label-only version:

\fbox{%
  \begin{minipage}{0.95\linewidth}
Find out the truth values of the conclusions based on the premises. The truth value can be True, False or Uncertain. Output the truth value in the format of "Truth value: " directly without any reasoning process.\\
Premises: \texttt{nl\_premises}\\
Conclusion: \texttt{conclusion}
  \end{minipage}
}

\subsection{Prompts used for proof evaluation}

\fbox{%
  \begin{minipage}{0.95\linewidth}
Given a deductive reasoning question, demonstrate whether the two reasoning chains are semantically similar and follow the same reasoning path to derive the final answer. After your explanations, output your decision in the format of "Decision: ". Your decision should be either Yes or No.\\
Premises: \texttt{nl\_premises}\\
Conclusion: \texttt{conclusion}\\
Reasoning chain A: \texttt{reasoning\_chain\_a}\\
Reasoning chain B: \texttt{reasoning\_chain\_b}
  \end{minipage}
}

\section{Details of Experiments}\label{appendix_detail_of_experiment}
In this section, we show the details of our experiments in Section \ref{section_experiments}.

\subsection{LLM Model version and cost}
We utilized API services provided by commercial companies to access these LLMs via the cloud services. The versions used are as follows.

\begin{itemize}
    \item GPT-4o: GPT-4o-2024-11-20 
    \item GPT-4.1: GPT-4.1-2025-04-14 
    \item Llama-3-70B: Llama-3.3-70B-Instruct 
    \item Deepseek-V3: Deepseek-V3-2025-03-24 
\end{itemize}

We use the default parameter to request all of the responses of these models. The total amount of the expenditure is approximately $\$$600.

\subsection{Natural language FOL reasoning}

For the PC-FOL dataset, we evaluate zero-shot (0-shot) and three-shot (3-shot) tasks on all tested LLMs. The few-shot examples are randomly selected from the dataset and include corresponding manually written proofs, each with a different story ID.

For the PC-FOL-Replace dataset, we similarly evaluate 0-shot and 3-shot tasks on all tested LLMs. The few-shot examples are randomly drawn from the PC-FOL dataset, accompanied by corresponding manually written proofs, and assigned a different story ID.


\subsection{Proof Evaluation}

\textbf{Pass@k Evaluation}

In this experiment, we sample $k$ different proofs by GPT-4o model.

\textbf{Few-shot prompting}

In this experiment, the few-shot prompting results are generated by GPT-4o. The evaluation LLM model in the ROUGE metric processing is set as GPT-4o.

\section{Experimental results for more models}\label{Appendix_more_experiments}

\subsection{Different Hyperparameters for Qwen3 Model}
In this section, we present the accuracy results for Qwen3 \citep{yang2025qwen3technicalreport} models with different hyperparameters and Deepseek-R1 models.  Table \ref{table_few_shot_nl_reasoning_eng_full} shows the accuracy results of zero-shot and few-shot prompting on PC-FOL dataset.

\begin{table}[!h]
  \caption{Logical reasoning accuracy results of zero-shot and few-shot prompting on English PC-FOL dataset.}
  \label{table_few_shot_nl_reasoning_eng_full}
  \centering
  \small
  \begin{tabular}{c|cccc}
    \toprule
    Model     & \makecell[l]{PC-FOL \\ Linear reasoning}  & \makecell[l]{PC-FOL \\ proof-by-case} & \makecell[l]{PC-FOL-Replace \\ Linear reasoning} & \makecell[l]{PC-FOL-Replace \\ proof-by-case} \\
    \midrule
    GPT-4o 0-shot & 85.13 & 51.08 & 84.34 & 51.66 \\
    GPT-4o 3-shot & 83.17 & 55.58 & 82.39 & 52.64 \\
    GPT-4.1 0-shot & 89.24 & 69.86 & 88.45 & 63.80 \\
    GPT-4.1 3-shot & 86.30 & 71.23 & 84.34 & 68.69 \\
    o4-mini 0-shot & 90.80 & 79.84 & 88.26 & 77.89 \\
    o4-mini 3-shot & 90.02 & 80.43 & 88.45 & 79.45 \\
    \midrule
    Llama 3-70B 0-shot & 80.63 & 46.77 & 79.84 & 50.88 \\
    Llama 3-70B 3-shot & 83.56 & 54.21 & 80.63 & 53.62 \\
    Deepseek-V3 0-shot & 89.63 & 76.71 & 86.69 & 73.39 \\
    Deepseek-V3 3-shot & 90.02 & 77.10 & 85.71 & 72.02 \\
    DeepSeek-R1 0-shot & 91.39 & 82.16 & 86.30 & 80.04 \\
    DeepSeek-R1 3-shot & 89.82 & 82.00 & 87.08 & 80.04 \\
    Qwen3-4B 0-shot & 85.32 & 61.45 & 83.76 & 55.77 \\
    Qwen3-4B 3-shot & 82.97 & 64.97 & 81.60 & 60.67 \\
    Qwen3-8B 0-shot & 84.93 & 61.64 & 84.93 & 59.69 \\
    Qwen3-8B 3-shot & 86.11 & 63.60 & 86.69 & 60.86 \\
    Qwen3-14B 0-shot & 87.28 & 63.60 & 84.54 & 59.10 \\
    Qwen3-14B 3-shot & 89.04 & 66.34 & 86.89 & 64.97 \\
    Qwen3-4B-think 0-shot & 87.67 & 74.56 & 85.52 & 72.21 \\
    Qwen3-4B-think 3-shot & 88.85 & 75.15 & 86.89 & 72.21 \\
    Qwen3-8B-think 0-shot & 89.43 & 77.30 & 84.93 & 73.97 \\
    Qwen3-8B-think 3-shot & 90.22 & 80.63 & 88.45 & 75.54 \\
    Qwen3-14B-think 0-shot & 91.39 & 80.82 & 87.87 & 77.10 \\
    Qwen3-14B-think 3-shot & 91.59 & 79.06 & 85.91 & 76.52 \\
    \bottomrule
  \end{tabular}
\end{table}

\subsection{Different random seeds or temperatures}
In this section, we present additional experiments evaluating the GPT-4o model on our dataset. Table \ref{table_appendix_gpt4o_random_seed_1} and Table \ref{table_appendix_gpt4o_temperature_1} show the results of the experiments across different random seeds and temperatures. These experiments demonstrate that varying random seeds or temperatures slightly affects the observed performance gap between linear-reasoning and proof-by-cases instances.

\begin{table}[!h]
      \caption{Evaluating the GPT-4o model on PC-FOL dataset. Random seeds of the model are fixed as 1.}
      \label{table_appendix_gpt4o_random_seed_1}
      \centering
      \small
      \begin{tabular}{c|cc>{\columncolor{gray!20}}c}
        \toprule
           \  & \makecell[l]{PC-FOL \\ Linear-Reasoning \\ Accuracy}  & \makecell[l]{PC-FOL \\ Proof-by-Cases \\ Accuracy} &\makecell[l]{PC-FOL \\ Acc. Gap}\\
        \midrule
        Temperature 0   & 82.39 & 54.99 & \textbf{27.40}{\small ↓}\\
        Temperature 0.5 & 83.17 & 55.58 & \textbf{27.59}{\small ↓}\\
        Temperature 1   & 82.39 & 52.45 & \textbf{29.94}{\small ↓}\\
        \midrule
        Average & 82.65 & 54.34 & \textbf{28.31}{\small ↓}\\
        \midrule
        \makecell[l]{Standard \\ Deviation} & 0.4503 & 1.6632 & - \\
        \bottomrule
      \end{tabular}
\end{table}

\begin{table}[!h]
      \caption{Evaluating the GPT-4o model on PC-FOL dataset. Temperatures of the model are fixed as 1.}
      \label{table_appendix_gpt4o_temperature_1}
      \centering
      \small
      \begin{tabular}{c|cc>{\columncolor{gray!20}}c}
        \toprule
           \  & \makecell[l]{PC-FOL \\ Linear-Reasoning \\ Accuracy}  & \makecell[l]{PC-FOL \\ Proof-by-Cases \\ Accuracy} &\makecell[l]{PC-FOL \\ Acc. Gap}\\
        \midrule
        Random seed 1   & 82.39 & 52.45 & \textbf{29.94}{\small ↓}\\
        Random seed 11 & 84.15 & 52.25 & \textbf{31.90}{\small ↓} \\
        Random seed 22   & 82.19 & 48.73 & \textbf{33.46}{\small ↓}\\
        \midrule
        Average & 82.91 & 51.14  & \textbf{31.77}{\small ↓}\\
        \midrule
        \makecell[l]{Standard \\ Deviation} & 1.0785 & 2.0924 & - \\
        \bottomrule
      \end{tabular}
\end{table}


\section{Fine-tuning Results}\label{appendidx_fine_tuning_results}

\subsection{Settings}
To evaluate the impact of proof-by-case data on improving reasoning capabilities, we train our models on two distinct dataset variants: one includes expert proofs in the completion block, while the other contains only ground-truth labels as target answers. The datasets are split by 70\%/15\%/15\% as mentioned in Section~\ref{section_experiments}.

We fine-tune an encoder-decoder model, Flan-T5-large\citep{flant5}, alongside 2 LLMs, including Llama-3.1-8B-Instruct\citep{grattafiori2024llama3herdmodels} and Qwen3-0.6B\citep{qwen2,qwen2.5}.
For the Flan-T5-large model, we conduct training on a single NVIDIA A100 GPU with a batch size of 8 and a learning rate of 1e-4.

For the remaining LLMs, we adopt a learning rate of 5e-5. Specifically, we train Llama3-8B and Qwen3-0.6B on 8 and 4 NVIDIA A800 GPUs, respectively, maintaining a consistent batch size of 16 across all experiments.

All the fine-tuning experiments utilize the AdamW optimizer \citep{loshchilov2017decoupled} with a cosine annealing learning rate scheduler \citep{loshchilov2016sgdr} and a warm-up ratio of 10\% during training.
Models are trained for 3 epochs.

\subsection{Results and Analysis}
Table \ref{table_fine_tuning} shows the fine-tuning results by using our PC-FOL dataset.

Flan-T5 exhibits consistent performance gains, while the accuracy improvement across LLMs ranges from 4\% to 9\%. However, we also observe a decline in label accuracy, where proof-by-case performance on Qwen3-0.6B even decreased by 6.58\%, in the with proof setting after fine-tuning, and the validation loss suggests overfitting.

To further investigate this unexpected phenomenon, we conduct an experiment where we train Qwen3-0.6B on the linear reasoning portion of the data and evaluate it on the proof-by-case portion—and vice versa. The results are shown in Table \ref{table_fine_tuning_qwen3}.

Our findings reveal that training on linear proof tokens negatively impacts performance in 
proof-by-case problems
, whereas models trained on proof-by-case data show significant improvement in simpler linear reasoning tasks.

This provides a plausible explanation for the lower proof-by-case accuracy in fine-tuned models:
when trained on mixed data, models tend to prioritize easily learnable patterns for simpler problems, which may hinder their ability to reason deeply in more complex scenarios.

Conversely, when proof procedures are omitted, models appear to "learn" label prediction superficially without truly "understanding" the underlying reasoning.

\begin{table}[!h]
  \caption{Fine-tuning results (Label Accuracy in \%) of LLMs trained on the full set of PC-FOL dataset and tested on the test set of Case/Linear.}
  \label{table_fine_tuning}
  \centering
  \small
  \begin{tabular}{c|c|cccc}
    \toprule
    \multirow{2}{*}{Model} & \multirow{2}{*}{\makecell{w/ SFT}} 
        & \multicolumn{2}{c}{\makecell[c]{PC-FOL w/ proof}} 
        & \multicolumn{2}{c}{\makecell[c]{PC-FOL w/o proof}} \\
    & & \makecell[c]{Linear reasoning} & \makecell[c]{proof-by-case} 
      & \makecell[c]{Linear reasoning} & \makecell[c]{proof-by-case} \\
    \midrule
    \multirow{2}{*}{Flan-T5-large}
        & Yes & 47.37\textcolor{ForestGreen}{\small ↑} & 35.53\textcolor{ForestGreen}{\small ↑} & 57.89\textcolor{ForestGreen}{\small ↑} & 47.37\textcolor{ForestGreen}{\small ↑} \\
        & No & 46.05 & 30.26 & 43.42 & 36.84 \\
    \midrule
    \multirow{2}{*}{Llama-3.1-8B-Instruct}
        & Yes &73.68\textcolor{ForestGreen}{\small ↑} & 48.68\textcolor{BrickRed}{\small ↓} & 59.21\textcolor{ForestGreen}{\small ↑} & 46.05\textcolor{ForestGreen}{\small ↑}\\
        & No &64.47 & 50.00 & 57.89 & 40.79\\
    \midrule
    \multirow{2}{*}{Qwen3-0.6B}
    & Yes & 61.84\textcolor{ForestGreen}{\small ↑} & 32.89\textcolor{BrickRed}{\small ↓} & 40.79\textcolor{NavyBlue}{\small -} & 36.84\textcolor{ForestGreen}{\small ↑}\\
    & No & 57.89 & 39.47 & 40.79 & 32.89\\
    \bottomrule
  \end{tabular}
\end{table}

\begin{table}[!h]
  \caption{Fine-tuning results of Qwen3 trained on the Linear/Case of PC-FOL dadaset and tested on the test set of Case/Linear.}
  \label{table_fine_tuning_qwen3}
  \centering
  \small
  \begin{tabular}{c|c|cccc}
    \toprule
    \multirow{2}{*}{Model} & \multirow{2}{*}{\makecell{SFT Data}}  & \multicolumn{2}{c}{\makecell[c]{PC-FOL w/ proof}} & \multicolumn{2}{c}{\makecell[c]{PC-FOL w/o proof}} \\
    & & \makecell[c]{Linear reasoning} & \makecell[c]{proof-by-case} 
      & \makecell[c]{Linear reasoning} & \makecell[c]{proof-by-case} \\
    \midrule
    \multirow{3}{*}{Qwen3-0.6B} & Linear & - & 36.84 & - & 47.37\\
    & Case & 72.37 & - &  51.32  & -\\
    & Null & 57.89 & 39.47 & 40.79 & 32.89\\
    \bottomrule
  \end{tabular}
\end{table}


\section{Proofs of the Theorems}\label{appendix_proof}

\subsection{Proof of Theorem \ref{theorem_6_1} for Linear-Reasoning problem}

Recall the assumptions and definitions:

\begin{enumerate}
    \item Define each $X_i$ as a proposition described in the given premises.
    \item The conditional probability $P[X_{i}|Context, X_1, \ldots, X_{i-1}]$ can be viewed as selecting the next proposition based on the augmented context.
    \item Fixed the LLM as a function $F$. Define $p_F(Augmented \ Context)$ as the correctness percentage of selecting the correct proposition from context. The values for different augmented contexts are assumed to be the same, defined as $p_F$.
\end{enumerate}

First, we give the explanations of the assumptions, and the reasons why these assumptions can hold.

\subsubsection{Define each $X_i$ as a proposition described in the given premises.}

Traditional analysis for the neural-based NLP models directly considered the probability distribution as $P[X_{1},\ldots,X_{n}]=\prod _{i=1}^{n}P[X_{i}|Context, X_1, \ldots, X_{i-1}]$. Although it is reasonable for analyzing LLMs, this method can not reveal the structure of one proposition (for example, “Grass is green”, “One is not a computer”) in a generated sentence. 

Therefore, to emphasize the logical relationships between the propositions themselves rather than the specific meaning of each word, we define $X_i$ as a proposition in the given premises. Under this definition, various equivalent natural language expressions can be represented by the same $X_i$, such as "$X_i$ is true" and "the truth value of $X_i$ is 1."

\subsubsection{The conditional probability $P[X_{i}|Context, X_1, \ldots, X_{i-1}]$ can be viewed as selecting the next proposition based on the augmented context.}

That is, we assumed that, for any prompt, with any variant of the sentence structure, which includes the needed context (the \textbf{full} given premises, the previous reasoning steps, the previous state, asking for the next reasoning step, etc.), the distribution of selecting the next proposition for the LLM is the same.

For example, suppose we have the premise `A is true, then B is true', and we have the fact (or in previous steps) that `A is true', and we want to ask the LLM for the next step or if B is true. The input can be ``Suppose we have a premise `A is true, then B is true', and we know that `A is true', what can we find based on these facts?'', or ``We have already known that `A is true, then B is true', and the previous theorem shows `A is true', will `B' be a correct proposition?''. 

Under this assumption, the augmented context contains all the needed information for the next-step reasoning.

\subsubsection{Fixed the LLM as a function $F$. Define $p_F(Augmented \ Context)$ as the correctness percentage of selecting the correct proposition from context. The values for different augmented contexts are assumed to be the same, defined as $p_F$.}

This assumption is based on a basic concept of a well-trained LLM: when giving full premises and the previous state, asking for the next reasoning step (must be linear connectivity), the accuracy of selecting the next step should be almost the same. Since all the training data are selected randomly and batch-wise, a well-trained LLM has such property theoretically.


\subsubsection{Proof}

Based on the previous assumptions, we can prove the theorem \ref{theorem_6_1}:

Under previous assumptions, for a LLM $F$, a linear-reasoning dataset $D$ with the distribution $P_{D}$ of reasoning steps, the probability that the LLM $F$ can give a correct proof is $\sum\limits_{k=1}^{\infty} p_F^k \cdot P_D(X=k)$, or $E_{X \sim P_D}[p_F^{|X|}]$

\begin{proof}
    From definition 1, the probabilistic graphical model $P[X_{1},\ldots ,X_{n}]=\prod _{i=1}^{n}P[X_{i}|{\text{pa}}(X_{i})]$ can be re-written as 
    \begin{equation}
        P[X_{1},\ldots ,X_{n}]=\prod _{i=1}^{n}P[X_{i}|Context, X_{1}, \cdots, X_{i-1})]
    \end{equation}
    for proposition $X_i$, and a linear-reasoning type step-by-step proof $X_1 \rightarrow X_2 \rightarrow \cdots \rightarrow X_{n}$.
    
    Then, by definition 2, the formula has the form of $P[X_{1},\ldots ,X_{n}]=\prod _{i=1}^{n}P[X_{i}|Augmented \ Context]$, where ``Augmented Context" represents a input with full premises and previous state.
    
    Consequently, by premise 3, since we assume that all the state transition has the same accuracy, we can get the probability formula of generating correct proof sequence $X_{1},\ldots ,X_{n}$ as 
    \begin{equation}
    \begin{split}
        P[X_{1},\ldots ,X_{n}]&=\prod _{i=1}^{n}P[X_{i}|Augmented \ Context]\\
                              &=\prod _{i=1}^{n}p_F(Augmented \ Context)\\
                              &= \prod _{i=1}^{n}p_F\\
                              &=p_F^n
    \end{split}
    \end{equation}

    As a result, when given a linear-reasoning dataset $D$ with the distribution $P_{D}$ of reasoning steps, the probability that the LLM $F$ can give a correct proof is
    \begin{equation}
    \begin{split}
        P_{data\sim D}(correct \ proof)
        &=\frac{1}{|D|}\sum\limits_{|D|}P(correct \ proof \ for \ data_i)\\
        &=\frac{1}{|D|}\sum\limits_{|D|}P(X_{i,1},\ldots ,X_{i,n_i})\\
        &=\frac{1}{|D|}\sum\limits_{|D|}p_F^{n_i} \ (\ n_i \ represents \ the \ number \ of \ reasoning \ steps)\\
        &=\frac{1}{|D|}\sum\limits_{k=1}^{\infty}p_F^k \cdot N_k \ (N_k \ represents \ the \ number \ of \ instances \ need \ k \ steps \ to \ prove)\\
        &=\sum\limits_{k=1}^{\infty}p_F^k \cdot \frac{N_k}{|D|}\\
        &=\sum\limits_{k=1}^{\infty} p_F^k \cdot P_D(X=k)\\
        &=E_{X \sim P_D}[p_F^{|X|}]\\
    \end{split}
    \end{equation}
\end{proof}

By using the theorem \ref{theorem_6_1}, and using the distribution of the proof steps in our dataset (shown in Appendix \ref{distribution_of_proof_steps}), we can calculate the probability $p_F$, the accuracy for getting the next step for different LLM $F$. Table \ref{figure_graphical_model_linear_reasoning} shows that the correct proof rate for GPT-4o over our dataset is 80.63$\%$, by using Theorem \ref{theorem_6_1}, we can calculate that $p_{GPT-4o}=92.62\%$. Thus, the probability of GPT-4o generating a correct proof for a linear-reasoning mathematical logic problem with $k$ steps is estimated by $0.9262^k$.

We compare the ground truth of the correctness ratio and the estimated correctness ratio $0.9262^k$ in Figure \ref{figure_comparison_correctness}. In this Figure, we can find that, for the same number of proof steps,the estimated correctness ratio is almost the same as the ground truth value. Therefore, we can conclude that our theory is suitable for analyzing the reasoning ability for LLM over linear-reasoning mathematical logic problems.

\begin{figure}[!h]
  \centering
  \includegraphics[width=\linewidth]{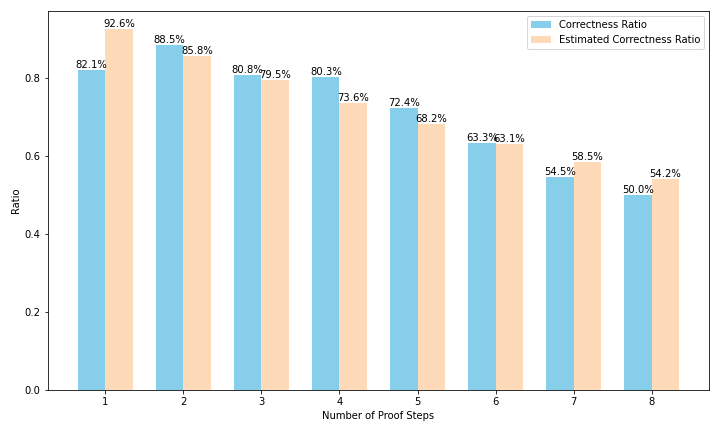}
  \caption{Comparison between the ground truth of the correct proof ratio and the estimated correctness ration over the PC-FOL Linear-Reasoning dataset.}
  \label{figure_comparison_correctness}
\end{figure}

\subsection{Proof of Theorems for Proof-by-Cases problem}\label{appendix_proof_proofbycases}

We give two examples to show the structure of the standard proof-by-cases problem.

\begin{figure}[!h]
  \centering
  \includegraphics[width=.45\textwidth]{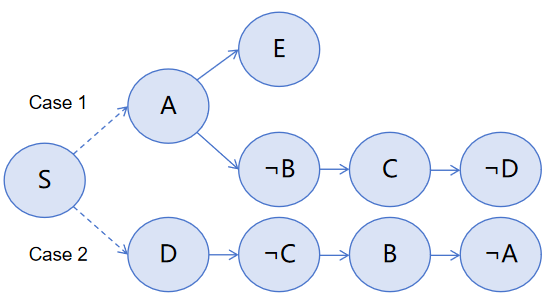}
  \hspace{1cm}
  \includegraphics[width=.45\textwidth]{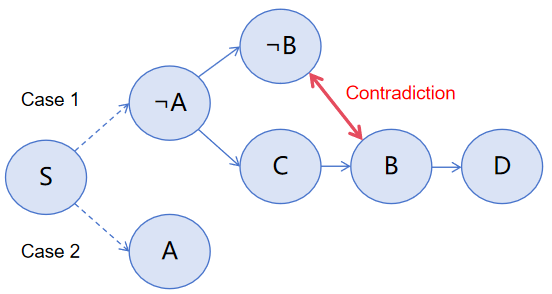}
  \caption{\textbf{Left}: The abstracted reasoning chain for right side example of Table \ref{table_of_difference_between_linear_and_case}. \textbf{Right}: The abstracted reasoning chain for a typical example that existing a contradiction in a subcase.}
  \label{figure_graphical_model_proof_by_case2}
\end{figure}

The first is the example shown in the right side of Table \ref{table_of_difference_between_linear_and_case}. The premises of the example are:
\begin{itemize}
    \item Any person that is tall does not major in physics. 
    \item All people who are not tall study quantum computing.
    \item All students major in math or physics. 
    \item If a person majors in math, then the person studies algebraic geometry. 
    \item If a person majors in math, then the person does not study quantum computing. 
    \item Billy is a student who studies algebraic geometry. 
\end{itemize}

When drawing the reasoning chain of this example, we define S as the proposition "Billy is a student, and is a person", A as "Billy majors in math", B as "Billy studies quantum computing", C as "Billy is tall", D as "Billy majors in physics", and E as "Billy studies algebraic geometry." Thus, the reasoning chain can be drawn as the left part in Figure \ref{figure_graphical_model_proof_by_case2}.

Another example we call it "subcase-contradiction", since there exists a contradiction when trying to get the truth value of some propositions in a subcase. The premises of the example are:

\begin{itemize}
    \item A neuroimaging technique is either an invasive neuroimaging technique or a noninvasive neuroimaging technique.
   \item All noninvasive neuroimaging techniques provide a spatial resolution of brains.
   \item If a technique provides a spatial resolution of brains, then it is a measurement of brain activity.
   \item All measurements of brain activity are used by neuroscience researchers.
   \item FMRI is either a measurement of brain activity or a noninvasive neuroimaging technique.
   \item FMRI is a neuroimaging technique.
\end{itemize}

When drawing the reasoning chain of this example, we define S as the proposition "FMRI is a neuroimaging technique", A as "FMRI is an invasive neuroimaging technique", B as "FMRI is a measurement of brain activity", C as "FMRI provides a spatial resolution of brains", and D as "FMRI is used by neuroscience researchers". Thus, the reasoning chain can be drawn as the right part in Figure \ref{figure_graphical_model_proof_by_case2}.

Recall the assumption that, for the $i$-th possible case, the reasoning chain for the given question is represented as $\{X^E_{i,1}, \cdots, X^E_{i,n_i}\}$, and for the $i$-th possible case, the reasoning chain for the given question is represented as $\{X^{NE}_{j,1}, \cdots, X^{NE}_{j,n_j}\}$. Under these definitions, the proof should be the set $\{\{X^{NE}_{i,1}, \cdots, X^{NE}_{i,n_i}\} \forall case \ i\in NE, \{X^{E}_{k,1}, \cdots, X^{E}_{k,n_i}\} for \ selected \  case \ k\in E\}$, where the $k$ are means that the $k$-th case of the set $E$ satisfies the given requisite of the question. Now we try to prove the Theorem \ref{theorem_6_2} and \ref{theorem_6_3}.

\subsubsection{Proof of Theorem \ref{theorem_6_2}}
Recall the theorem: 
Under previous assumptions, for a LLM $F$, a dataset $D$ with the distribution $P_{D}$ of the set of reasoning chains. Define $p_{F,cases}$ as the probability that the LLM can correctly identify all the cases required for discussion, then the probability that the LLM $F$ can give a correct proof is $\sum p_{F,cases}(X)\cdot p_F^{\sum_{i} |X^{NE}_i|+|X^{NE}|+\sum\limits_{selected \ k}|X^{E}_k|} \cdot P_D(X=\{all \{X^{NE}\}, selected \{X^{E}\}\})$

\begin{proof}
From the proof of Theorem \ref{theorem_6_1}, we have proven that for a linear-reasoning sequence $X_{i,1},\ldots ,X_{i,n_i}$, the probability of generating such a sequence correctly is $p_F^{n_i}$. 
Thus, to generate a proof with propositions $\{\{X^{NE}_{i,1}, \cdots, X^{NE}_{i,n_i}\} \forall case \ i\in NE, \{X^{E}_{k,1}, \cdots, X^{E}_{k,n_i}\} for \ selected \  case \ k\in E\}$, the number of reasoning steps should be $$\sum_{i} |X^{NE}_i|+|X^{NE}|+\sum\limits_{selected \ k}|X^{E}_k|.$$ 

Note that in order to show the cases in set $NE$ are impossible to exist, the proof should add an extra reasoning step to find the contradiction in such cases. That is the reason why we add $|X^{NE}|$ in the number of reasoning steps.

Hence, we have
\begin{equation}
    \begin{split}
        P_{data\sim D}(correct \ proof)
        &=\frac{1}{|D|}\sum\limits_{|D|}P(correct \ proof \ for \ data_i)\\
        &=\frac{1}{|D|}\sum\limits_{|D|}p_{F,cases}(X) \cdot \\
        &\ \ \ \ \ \ \ \ \ \ \ P(\{\{X^{NE}_{i,1}, \cdots, X^{NE}_{i,n_i}\} \forall case \ i\in NE, \{X^{E}_{k,1}, \cdots, X^{E}_{k,n_i}\} for \ selected \  case \ k\in E\})\\
        &=\frac{1}{|D|}\sum\limits_{|D|}p_{F,cases}(X) \cdot p_F^{\sum_{i} |X^{NE}_i|+|X^{NE}|+\sum\limits_{selected \ k}|X^{E}_k|} \\
        &=\sum p_{F,cases}(X)\cdot p_F^{\sum_{i} |X^{NE}_i|+|X^{NE}|+\sum\limits_{selected \ k}|X^{E}_k|} \cdot P_D(X=\{all \{X^{NE}\}, selected \{X^{E}\}\})
    \end{split}
\end{equation}

\end{proof}

\subsubsection{Proof of Theorem \ref{theorem_6_3}}
Recall the theorem:
If people define a correct proof as the set of all the reasoning steps in different cases, even for the cases that do not satisfy the requisite of the question, then under previous assumptions, for a LLM $F$, a dataset $D$ with the distribution $P_{D}$ of the set of reasoning chains. Then the probability that the LLM $F$ can give a correct proof is $\sum\limits_{k=1}^{\infty} p_F^k \cdot P_D(|X|+|X^{NE}|=k)$, or $E_{X \sim P_D}[p_F^{|X|+|X^{NE}|}].$

\begin{proof}
From the proof of Theorem \ref{theorem_6_1}, we have proven that for a linear-reasoning sequence $X_{i,1},\ldots ,X_{i,n_i}$, the probability of generating such a sequence correctly is $p_F^{n_i}$. 
To generate a proof with full propositions and logical reasoning steps $\{\{X^{NE}_{i,1}, \cdots, X^{NE}_{i,n_i}\} \forall case \ i\in NE, \{X^{E}_{k,1}, \cdots, X^{E}_{k,n_i}\} \forall case \ k\in E\}$, the number of reasoning steps should be $$\sum_{i} |X^{NE}_i|+|X^{NE}|+\sum_k|X^{E}_k|=|X^{NE}|+|X|.$$ 

Note that in order to show the cases in set $NE$ are impossible to exist, the proof should add an extra reasoning step to find the contradiction in such cases. That is the reason why we add $|X^{NE}|$ in the number of reasoning steps.

Hence, we have
\begin{equation}
    \begin{split}
        P_{data\sim D}(correct \ proof)
        &=\frac{1}{|D|}\sum\limits_{|D|}P(correct \ proof \ for \ data_i)\\
        &=\frac{1}{|D|}\sum\limits_{|D|}P(\{\{X^{NE}_{i,1}, \cdots, X^{NE}_{i,n_i}\} \forall case \ i\in NE, \{X^{E}_{k,1}, \cdots, X^{E}_{k,n_i}\}\forall case \ k\in E\})\\
        &=\frac{1}{|D|}\sum\limits_{|D|}p_F^{|X^{NE}|+|X|} \\
        &=\frac{1}{|D|}\sum\limits_{k=1}^{\infty}p_F^k \cdot N_{|X^{NE}|+|X|=k} \\ &\ \ \ \ \ (N_{|X^{NE}|+|X|} \ represents \ the \ number \ of \ instances \ need \ |X^{NE}|+|X|=k \ steps \ to \ prove)\\
        &=\sum\limits_{k=1}^{\infty}p_F^k \cdot \frac{N_{|X^{NE}|+|X|}=k}{|D|}\\
        &=\sum\limits_{k=1}^{\infty} p_F^k \cdot P_D(|X^{NE}|+|X|=k)\\
        & = E_{X \sim P_D}[p_F^{|X|+|X^{NE}|}]\\
    \end{split}
\end{equation}

\end{proof}

\section{Discussion about Proof-by-Contradiction technique}\label{appendix_detail_proof_by_contradiction}

Mathematically, Proof-by-Contradiction technique can be considered a sub-category of Proof-by-Cases, since all such proofs can be reformulated as a standard Proof-by-Cases proof by the following process.

The “Proof by Contradiction” proof involves:

\begin{enumerate}
    \item Assume the opposite of what people want to prove.
    \item Reason logically from this assumption.
    \item Find a contradiction in the reasoning steps.
    \item Conclude that the assumption must be false.
\end{enumerate}

Then, reformulated as 

\begin{enumerate}
    \item Consider two cases: what people want to prove is False and what people want to prove is True.
    \item In case 1, reason logically from the assumption that “what people want to prove is False”.
    \item Find a contradiction in the reasoning steps of case 1, and then conclude that case 1 can not occur.
    \item Show that case 2, “what people want to prove is true,” is compatible with the given premises, and conclude that only case 2 exists, which means that what people want to prove is True.
\end{enumerate}

Consequently, “Proof by Contradiction” can be viewed as a sub-category of “Proof by Cases.” Therefore, all instances in our dataset utilizing this technique are classified under the “Proof by Cases” category.

\section{Distribution of the number of proof steps}\label{distribution_of_proof_steps}

The distribution of the number of proof steps in each linear-reasoning instance of the PC-FOL dataset is shown in Figure \ref{figure_distribution_premises}.

\begin{figure*}[!h]
  \centering
  \includegraphics[width=\linewidth]{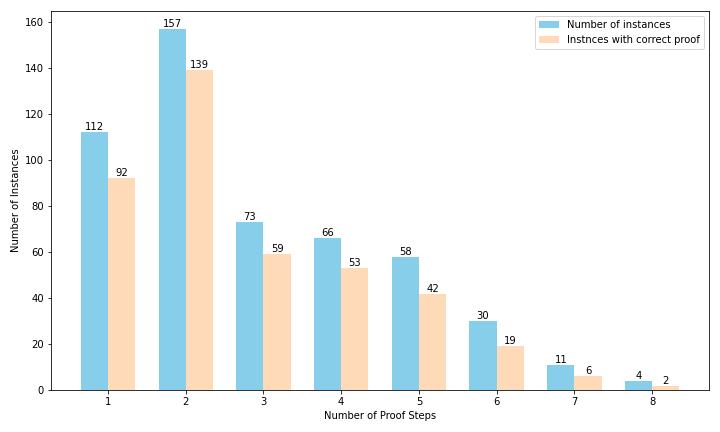}
  \caption{The distribution of the number of proof steps over the PC-FOL dataset (linear-reasoning part).}
  \label{figure_distribution_number_of_steps_linear_reasoning}
\end{figure*}

\end{document}